\documentclass{opt2025} 

\usepackage{csquotes}
\usepackage{ifthen}
\usepackage{enumitem}
\usepackage{amsmath,amssymb,amsfonts}
\usepackage{bm}
\usepackage{graphicx}
\usepackage{xspace}
\usepackage{verbatim}
\usepackage{algorithm}
\usepackage{algpseudocode}
\usepackage{cleveref}
\usepackage{xfrac}
\usepackage{bbm}

\usepackage{color}
\usepackage{thm-restate}
\usepackage{latexsym}

\allowdisplaybreaks
\DeclareMathOperator{\argmax}{argmax}
\newtheorem{assumption}[theorem]{Assumption}

\newcommand{\cA}{\mathcal{A}}

\newcommand{\cS}{\mathcal{S}}

\renewcommand{\epsilon}{\varepsilon}

\title[Central Limit Theorems for Asynchronous Averaged Q-Learning]{Central Limit Theorems for Asynchronous Averaged Q-Learning}

\optauthor{
\Name{Xingtu Liu} \Email{rltheory@outlook.com}\\
\addr Simon Fraser University}

\begin{document}

\maketitle

\begin{abstract}
This paper establishes central limit theorems for Polyak–Ruppert averaged Q-learning under asynchronous updates. We prove a non-asymptotic central limit theorem, where the convergence rate in Wasserstein distance explicitly reflects the dependence on the number of iterations, state–action space size, the discount factor, and the quality of exploration. In addition, we derive a functional central limit theorem, showing that the partial-sum process converges weakly to a Brownian motion.
\end{abstract}

\section{Introduction}

Reinforcement Learning (RL) has emerged as a powerful paradigm in artificial intelligence, achieving successes in various applications such as Atari \citep{mnih2015human}, Go \citep{silver2016mastering}, robot manipulation \citep{tan2018sim,zeng2020tossingbot}, and aligning large language models to human preferences \citep{shao2024deepseekmath,ouyang2022training}. Q-learning \citep{watkins1992q}, which directly learns the optimal action-value function (Q-function) from experience trajectories, is one of the most widely used RL algorithms.

Stochastic approximation (SA) \citep{benveniste2012adaptive, borkar2008stochastic} is a general iterative framework to solve fixed-point equation problems. Since the Bellman operator in RL is a contraction map with a unique fixed point, many RL algorithms can be interpreted as instances of SA. For example, TD learning \citep{sutton1998reinforcement} can be viewed as an instance of linear SA. Synchronous Q-learning, by contrast, is a special case of nonlinear SA with martingale noise. The asynchronous Q-learning algorithm studied in this work, however, is a nonlinear SA problem with Markovian noise. There is a growing line of work on finite-sample analysis of SA with applications to RL algorithms \citep{Wai19CC, srikant2019finite, bhandari2018finite, chen2022finite, chen2020finite, qu2020finite, khodadadian2025general, chandak2025finite, chandak20251}.

Polyak-Ruppert averaging is a classical variance-reduction technique to stabilize and accelerate SA algorithms. A key motivation for focusing on Polyak-Ruppert averaging is its dual advantage of practical robustness and statistical efficiency. Standard stochastic approximation algorithms are often sensitive to the specific choice of decaying stepsize, requiring hyperparameter tuning to achieve stable convergence. By contrast, averaging the iterates makes the algorithm significantly more robust to the underlying stepsize schedule. Furthermore, Polyak-Ruppert averaging is known to achieve the optimal rate of convergence by minimizing the asymptotic covariance matrix of the estimates. In this paper, we are interested in establishing central limit theorems (CLTs) for Polyak-Ruppert averaged Q-learning under asynchronous updates. Building CLTs provides a foundational understanding of the algorithm's statistical properties. This asymptotic normality is crucial for uncertainty quantification and statistical inference in RL. Building on the seminal work by \citet{polyak1992acceleration}, a non-asymptotic CLT for Polyak-Ruppert averaged SGD was established \citep{anastasiou2019normal}. \citet{mou2020linear,samsonov2025statistical} derive non-asymptotic CLTs for linear SA with Polyak–Ruppert
averaged iterates. Similar results for two-time-scale SA are also studied \citep{han2024decoupled, hu2024central, kong2025nonasymptotic}. Recently, CLTs for SA with applications to RL algorithms are studied \citep{borkar2025ode}. As a special case linear SA, \citet{srikant2024rates, samsonov2024gaussian} derive non-asymptotic CLTs for TD-learning with averaging. However, non-asymptotic CLTs for Q-learning remain unexplored.

As a special case of nonlinear SA, Q-learning is substantially more challenging to analyze than linear SA and TD learning. Functional CLTs for Polyak–Ruppert averaged synchronous Q-learning was established in \citet{xie2022statistical, li2023statistical, panda2025asymptotic}. Synchronous Q-learning only considers martingale noises. By contrast, asynchronous Q-learning updates a single state–action pair based on one transition sample at each iteration, which involves Markovian noises that are non-IID. Moreover, the empirical Bellman operator in asynchronous Q-learning is non-smooth. Thus, the challenges in analyzing asynchronous Q-learning come from nonlinearity, Markovian samples, and a non-smooth operator. Recently, \citet{zhang2024constant} established a functional CLT for asynchronous Q-learning with a constant stepsize. Constant stepsize does not satisfy the necessary conditions for establishing a non-asymptotic CLT, which we detail in Section \ref{sec_main}. To the best of our knowledge, no non-asymptotic CLT is currently known for Q-learning, even in the synchronous setting. In this work, we close this gap and prove both a non-asymptotic CLT and a functional CLT for asynchronous averaged Q-learning with decaying stepsizes.

\section{Preliminaries}

An infinite-horizon discounted Markov decision process (MDP) is denoted by $\mathcal{M}$, and is defined by the tuple $\langle \cS, \cA, P, r, \gamma \rangle$ where $\cS$ is the set of states, $\cA$ is the action set, $P: \cS \times \cA \rightarrow \Delta_\cS$ is the transition probability function, and $\gamma \in [0, 1)$ is the discount factor. Let $\Delta_{\cA}$ denotes the simplex over the action space. The action-value function (Q-function) of a stationary and stochastic policy $\pi: \cS \rightarrow \Delta_{\cA}$ is defined as $Q^{\pi}(s,a) = \mathbb{E} \Big[\sum_{t=0}^\infty \gamma^t r(s_t, a_t)|s_0=s,a_0=a\Big]$, where $a_t \sim \pi( \cdot | s_t)$ and $s_{t+1} \sim P( \cdot | s_t, a_t)$. The optimal Q-function is defined as $Q^*:=\max_{\pi}Q^{\pi}$. The value function is defined as $V^{\pi} = \pi Q^{\pi}$, where $(\pi Q)(s) :=\langle \pi(\cdot | s), Q(s,\cdot) \rangle$. We also define $P^{\pi} \in \mathbb{R}^{|\mathcal{S}||\mathcal{A}|\times |\mathcal{S}||\mathcal{A}|}$ such that $P^{\pi}Q=P(\pi Q)$. We make the following assumption over a specific optimal policy.
\begin{assumption}
\label{asm_1}
There exists an optimal policy $\pi^*$ such that for $Q \in \mathbb{R}^{|\mathcal{S}|\times|\mathcal{A}|}$ we have $\| (P^{\pi}-P^{\pi^*})(Q-Q^*)\|_{\infty} \leq L\|Q-Q^*\|^2_{\infty}$ where $\pi(s):=\argmax_{a\in\mathcal{A}}Q(s,a)$.
\end{assumption}

Adopted from \citet{li2023statistical}, this assumption provides a localized smoothness condition required to establish asymptotic normality for Polyak-Ruppert averaging. Unlike linear stochastic approximation, the Bellman optimality operator is non-smooth due to the $\max$ operation, meaning the greedy policy $\pi$ can shift abruptly near $Q^*$. Assumption \ref{asm_1} addresses this by imposing a margin condition, where the difference in transition dynamics between the current and optimal policies must shrink at a quadratic rate relative to the estimation error, $||Q - Q^*||_\infty$. This quadratic decay ensures that non-linear residual terms vanish asymptotically, allowing the leading-order martingale noise to dictate the covariance of the averaged iterates.

The asynchronous Q-learning algorithm maintains a Q-function estimator $Q_k$ and the update rule is the following:
\begin{align}
Q_{k+1} = Q_{k} + \alpha_k(F_k - Q_k) \label{update_Q}
\end{align}
where we let $F_k = F(Q_k, y_k)$, $y_k = (s_k,a_k,s_{k+1})$,
\begin{align}
[F(Q_k, s_k,a_k,s_{k+1})](s,a) = \mathbbm{1}_{\{ (s_k,a_k)=(s,a) \}} \Gamma(Q_k,s_k,a_k,s_{k+1}) + Q_k(s,a), \label{update_F}
\end{align}
and
\begin{align*}
\Gamma(Q_k,s_k,a_k,s_{k+1}) = r_k(s_k,a_k) + \gamma \max_{a} Q_k(s_{k+1},a) - Q_k(s_k,a_k).
\end{align*}
$\Gamma$ is the temporal difference in the Q-function iterate. The sample trajectory $\{(s_k,a_k)\}$ is collected by the MDP under a behavior policy $\pi_b$. We define $V_k(s) := \max_a Q_k(s,a)$. Now we make the following assumption on the Markov chain, which is standard in the literature \citep{zhang2024constant,zhang2022global, chen2021lyapunov, li2020sample, qu2020finite}.

\begin{assumption}
\label{asm_2}
$\{y_k\}_{k\geq0}$ is an irreducible and aperiodic finite state Markov chain $\mathcal{M}$.
\end{assumption}
Under Assumption~\ref{asm_2}, the Markov chain $\mathcal{M}$ admits a unique stationary distribution $\tilde{\mu}$. We denote $\tilde{S}$ as the state-space and $\tilde{P}$ as the transition kernel. Next, we define the Bellman operator for the Q-function:
\begin{align*}
[\mathcal{T}(Q)](s,a) = r(s,a) + \gamma \mathbb{E}_{s^{\prime}\sim P(\cdot |s,a)} \max_{a^{\prime}\in \mathcal{A}} Q(s^{\prime},a^{\prime}).
\end{align*}
Define $\pi_k$ such that $\mathcal{T}(Q_k) = r + \gamma \pi_k Q_k$. Denote by $\bar{F}_k$ the expected value of $F(Q_k,y_k)$, i.e. $\bar{F}_k := \bar{F}(Q_k) :=  \mathbb{E}_{y_k \sim \tilde{\mu}}[F(Q_k,y_k)]$. Further, denote by $D \in \mathbb{R}^{|\mathcal{S}||\mathcal{A}|\times |\mathcal{S}||\mathcal{A}|}$ the diagonal matrix with $\{ p(s,a)\}_{(s,a)\in \mathcal{S}\times \mathcal{A}}$ on its diagonal, where $p(s,a)$ is the stationary visitation probability of the state-action pair $(s,a)$. We denote $\rho:=\underset{(s,a)\in\mathcal{S}\times \mathcal{A}}{\min} p(s,a)$, which captures the quality of exploration. The following lemma is a consequence of Assumption \ref{asm_2}.

\begin{lemma}[Proposition 3.1 in \citet{chen2021lyapunov}] Suppose that Assumption \ref{asm_2} holds, we have
$$\bar{F}(Q) = D \mathcal{T}(Q) + (I-D) Q.$$
\end{lemma}
We denote the Markov chain mixing time at iteration $k$ as $t_k$. Formally, the mixing time $t_k$ of the Markov chain $\mathcal{M}$ is defined as $t_k := \min \{ i\geq0: \max_{s\in\tilde{S}} \|\tilde{P}^{i}(s,\cdot) - \tilde{\mu}(\cdot) \|_{\text{TV}} \leq \alpha_k \}$.

\section{Main Results}
\label{sec_main}

In this section, we present our main results. We first establish a non-asymptotic CLT for the averaged Q-learning iterates, providing an explicit rate at which their distribution approaches a normal distribution. The deviation is measured by using the 1-Wasserstein distance. We then derive a functional central limit theorem (FCLT), showing that the partial-sum process converges weakly to a Brownian motion.

\subsection{Non-Asymptotic Central Limit Theorem}

Let $\Delta_k = Q_k - Q^*$. Our goal is to study the rate at which $\frac{1}{\sqrt{K}}\sum_{k=1}^K \Delta_k$ converges in distribution to normality. We present the main result as follows, where we use big $O$ notation to hide all constants.
\begin{theorem}
\label{thm_CLT}
Let $\alpha_k=\alpha(k+b)^{-\beta}$ for some constants $\alpha,b>0$ and $\beta \in(0.5,1)$. Under Assumption \ref{asm_1} and \ref{asm_2}, we have the following rate of convergence 
\begin{align*}
\mathcal{W}_1 \left( K^{-\frac{1}{2}} \sum_{k=1}^K\Delta_{k}, \tilde{\mathcal{N}} \right) &\leq \frac{({|\mathcal{S}||\mathcal{A}|})^{\frac{1}{2}}}{\rho(1-\gamma)^2{K}^{\frac{1}{2}}} \cdot \tilde{O}\left( (\rho(1-\gamma))^{\frac{\beta-2}{1-\beta}} + K^{\beta/2}\rho^{-1}(1-\gamma)^{-1}  \right.\\
&\left. \indent + K^{1-\beta} + K^{\frac{1-\beta}{2}}\rho^{-1-\beta}(1-\gamma)^{-\beta} \right)
\end{align*}
where $\tilde{\mathcal{N}}=(A^{-1}\Sigma A^{-\top})^{1/2}\mathcal{N}(0,I)$, $A=D-\gamma DP^{\pi^*}$, $\Sigma := \sum_{i,j\in \tilde{S}} \tilde{\mu}(i)  \tilde{P}(i,j) (X(j) - \mathbb{E}[X(Y_1)|Y_0=i])(X(j) - \mathbb{E}[X(Y_1)|Y_0=i])^{\top}$ and $X$ is the solution to a Poisson's equation.
\end{theorem}
We now derived a non-asymptotic CLT showing that the distribution of the algorithm's average error converges towards a normal distribution. The asymptotic covariance matrix $\Sigma$ describes the variance in the learning process that comes from sampling transitions from the environment. The asynchronous Q-learning updates are noisy because they are based on single transition samples, which is not IID. The matrix $\Sigma$ quantifies the long-term structure of this randomness. The parameter $\rho$ quantifies the quality of exploration. Recall that $\tilde{u}, \tilde{P},$ and $\tilde{S}$ are the stationary distribution, state-space, and transition kernel of the Markov chain $\mathcal{M}$. We define $X: \tilde{S}\rightarrow \mathbb{R}^{|\mathcal{S}|\times|\mathcal{A}|}$ to be the solution of the Poisson's equation: $F(Q^*, i) - \mathbb{E}[F(Q^*, i)] = X(i) - \mathbb{E}[X(Y_1)| Y_0=i] \ \ \forall \ \ i\in \tilde{S}$.

The stepsize in the Q-learning update chosen in this work is $\alpha_k=\alpha(k+b)^{-\beta}$ for two reasons. First, convergence of stochastic approximation with averaging schemes relies on several key conditions \citep{polyak1992acceleration,li2023statistical}: (i) $0\leq \sup_k \alpha_k \leq1$, $\alpha_k \downarrow 0$ and $k\alpha_k \uparrow \infty$; (ii) $\frac{\alpha_{k-1}-\alpha_k}{\alpha_{k-1}}=o(\alpha_{k-1})$; (iii) $\frac{1}{\sqrt{K}}\sum_{k=0}^K \alpha_k \rightarrow 0$; (iv) $\frac{\sum_{k=0}^K \alpha_k}{K\alpha_K}\leq C$. Constant stepsizes violate conditions (i) and (iii), while linear stepsizes violate condition (ii). By contrast, polynomial stepsizes satisfy all of the above. Second, the problem-dependent constants $\alpha$ and $b$ are crucial for establishing a finite-sample convergence guarantee \citep{chen2021lyapunov}, which we leverage in our analysis. The parameter $\alpha$ acts as a scaling factor for balancing the trade-off between the speed of convergence and the final error of the algorithm. The parameter $b$ is used to control the magnitude of the initial stepsizes and ensure the stability of the algorithm during the early stages.
Setting $\beta=2/3$, the rate of convergence can be simplified as follows.
\begin{corollary}
Under Assumption \ref{asm_1}, \ref{asm_2}, and with $\alpha_k=\alpha(k+b)^{-\frac{2}{3}}$, $K\geq (\rho(1-\gamma))^{-12}$, we have 
\begin{align*}
&\mathcal{W}_1 \left( K^{-\frac{1}{2}} \sum_{k=1}^K\Delta_{k}, (A^{-1}\Sigma A^{-\top})^{1/2}\mathcal{N}(0,I) \right) \leq \tilde{O}\left( \frac{({|\mathcal{S}||\mathcal{A}|)^{\frac{1}{2}}}}{K^{\frac{1}{6}}\rho^2(1-\gamma)^3} \right).
\end{align*}
\end{corollary}

\subsection{Proof Sketch of Theorem \ref{thm_CLT}}

The proof of the non-asymptotic central limit theorem for asynchronous averaged Q-learning addresses the core challenges of nonlinearity, the non-smoothness of the empirical Bellman operator, and non-IID Markovian noises. The proof contains five main steps.

\paragraph{Step 1: Constructing Bounding Processes} Note that the empirical Bellman operator in asynchronous Q-learning is non-smooth due to the $\max$ operator and the indicator function. The proof introduces upper and lower bounding processes, denoted as $\Delta_k^\uparrow$ and $\Delta_k^\downarrow$, to track the error $\Delta_k = Q_k - Q^*$. By carefully bounding the update steps using the properties of the greedy and optimal policies, it is established by induction that $\Delta_k^\downarrow \le \Delta_k \le \Delta_k^\uparrow$ holds for all $k \in [K]$.

\paragraph{Step 2: Recursive Error Decomposition} We first analyze the upper bounding process, where the accumulated error sum $\sum_{k=1}^K \Delta_k^\uparrow$ is expanded recursively. We decompose the sum into five primary terms. These terms isolate distinct sources of error: the initialization bias, the smooth temporal difference errors, and the stochastic noise term.

\paragraph{Step 3: Handling Markovian Noise via Poisson Equation} Asynchronous updates rely on single transition samples, which introduces Markovian noise into the sequence. To address this lack of independence, the proof leverages the Poisson equation technique \citep{glynn1996liapounov, douc2018markov, makowski2002poisson, chen2020explicit}. By applying the solution to the Poisson equation, the non-stationary Markovian noise is rewritten as the sum of a bounded martingale difference sequence and a telescoping-like correction term. This cleanly isolates the true martingale noise required for the central limit theorem.

\paragraph{Step 4: Bounding the Residual Components} With the error decomposed, the non-martingale residual terms, such as the initialization decay and the Poisson equation correction terms, are bounded in the infinity norm. The analysis demonstrates that when these residual terms are scaled by $1/\sqrt{K}$, they vanish asymptotically at a fast enough rate and do not affect the limiting distribution.

\paragraph{Step 5: Applying the Martingale CLT and Sandwiching} After bounding all remainder terms, a non-asymptotic martingale central limit theorem is applied directly to the isolated martingale difference sequence. This establishes that the normalized sum of the upper bounding process converges to the normal distribution in the 1-Wasserstein distance. Finally, since the lower bounding process $\Delta_k^\downarrow$ follows an analogous decomposition and converges to the exact same distribution, a sandwich argument concludes that the true averaged Q-learning error $\frac{1}{\sqrt{K}}\sum_{k=1}^K \Delta_k$ inherits this identical convergence rate.

\subsection{Functional Central Limit Theorem}

The FCLT is an important extension to the conventional CLT. Donsker’s FCLT \citep{donsker1951invariance} states that the normalized partial sum process of i.i.d. random variables converges weakly to a Brownian motion in the Skorokhod space. In this section, we establish an FCLT for asynchronous Q-learning iterates, showing that the partial-sum process converges in distribution to a rescaled Brownian motion. Let $\mathcal{D}[0,1]$ denote the Skorokhod space. For $\zeta \in [0,1]$, we define the standardized partial sum processes associated with $\{ Q_k \}_{k\geq 1}$ as
\begin{align*}
\Phi_{K}(\zeta) = \frac{1}{\sqrt{K}}\sum_{k=1}^{\lfloor \zeta K \rfloor} \Delta_k = \frac{1}{\sqrt{K}}\sum_{k=1}^{\lfloor \zeta K \rfloor} (Q_k - Q^*).
\end{align*}

\begin{theorem}
\label{thm_FCLT}
Under the setting of Theorem \ref{thm_CLT}, the partial sum process $\Phi_{K}(\cdot)$ converges weakly to $(A^{-1}\Sigma A^{-\top})^{1/2} \textbf{B}(\cdot)$ on $\mathcal{D}[0,1]$, where $\textbf{B}(\cdot)$ is the standard Brownian motion on $[0,1]$.
\end{theorem}
We can see that the conventional CLT is a special case of the FCLT when $\zeta=1$. As the FCLT provides a basis for the asymptotic normality of certain functionals of stochastic processes, it is important for uncertainty quantification and statistical inference in Q-learning. Previous works have established the FCLT for synchronous Q-learning \citep{xie2022statistical, li2023statistical, panda2025asymptotic}. A recent work \citep{zhang2024constant} established a FCLT for asynchronous Q-learning with a constant step size. In contrast, our result concerns diminishing step-sizes.

\section{Conclusion}

We present a non-asymptotic central limit theorem for asynchronous averaged Q-learning, showing that the averaged iterate converges to a normal distribution in the Wasserstein distance at a rate of $\tilde{O}\left( {({|\mathcal{S}||\mathcal{A}|})^{\frac{1}{2}}}{K^{-\frac{1}{6}}\rho^{-2}(1-\gamma)^{-3}} \right)$. We also derive a functional CLT, showing weak convergence of the partial-sum process to a Brownian motion. Compared with linear stochastic approximation and TD learning, the analysis of Q-learning poses additional challenges due to its nonlinearity and the non-stationarity of the process. Asynchronous updates further complicate the problem by introducing Markovian noise. This work identifies and addresses all of these challenges to provide the first non-asymptotic CLT for Q-learning. An important future direction is to strengthen the convergence rate and to extend the results to other metrics beyond the 1-Wasserstein distance.

\bibliography{ref}

\appendix

\newpage

\section{Proof of Theorem \ref{thm_CLT}}

We begin by establishing the following lemma. 

\subsection{Proof of Lemma \ref{lem_1}}

\begin{lemma}
\label{lem_1}
Denote $\Delta_k = Q_k - Q^*$. For all $k \in [K]$, if $\alpha_k \leq 1$, then $\Delta_k$ is bounded as follows:
\begin{align*}
\Delta^{\downarrow}_k \leq \Delta_k \leq \Delta^{\uparrow}_k,
\end{align*}
where $\Delta^{\downarrow}_0 = \Delta_0 = \Delta^{\uparrow}_0$ and the upper and lower bounds evolve according to
\begin{align*}
\Delta^{\uparrow}_{k+1} = (I - \alpha_kD + \alpha_k \gamma D P^{\pi^*})\Delta^{\uparrow}_k + \alpha_k \gamma D(P^{\pi_k}-P^{\pi^*} )\Delta_k  + \alpha_k(F_k - \bar{F}_k),
\end{align*}
and
\begin{align*}
\Delta^{\downarrow}_{k+1} = (I - \alpha_kD + \alpha_k \gamma D P^{\pi^*})\Delta^{\downarrow}_k  + \alpha_k(F_k - \bar{F}_k).
\end{align*}
\end{lemma}

\begin{proof}
We first show that
\begin{align}
\Delta_{k+1} = (I - \alpha_kD + \alpha_k \gamma D P^{\pi^*})\Delta_k + \alpha_k \gamma D(P^{\pi_k}-P^{\pi^*} )Q_k  + \alpha_k(F_k - \bar{F}_k). \label{eqn_1}
\end{align}
By the asynchronous Q-learning update rule, we have
\begin{align*}
Q_{k+1} &= Q_{k} + \alpha_k(F_k - Q_k)\\
&= Q_k + \alpha_k(\bar{F}_k - Q_k) + \alpha_k(F_k - \bar{F}_k)\\
&= Q_k + \alpha_k(D \mathcal{T}(Q_k) + (I-D) Q_k - Q_k) + \alpha_k(F_k - \bar{F}_k)\\
&= Q_k + \alpha_k(D \mathcal{T}(Q_k) - DQ_k) + \alpha_k(F_k - \bar{F}_k)\\
&= Q_k + \alpha_kD(\mathcal{T}(Q_k) - Q_k) + \alpha_k(F_k - \bar{F}_k).
\end{align*}
Subtracting $Q^*$ from both sides yields
\begin{align*}
Q_{k+1} - Q^* &= Q_k + \alpha_kD(\mathcal{T}(Q_k) - Q_k) + \alpha_k(F_k - \bar{F}_k) - Q^*\\
&= (I - \alpha_kD) Q_k + \alpha_kD\mathcal{T}(Q_k)  + \alpha_k(F_k - \bar{F}_k) - Q^*\\
&= (I - \alpha_kD) (Q_k - Q^*)+ \alpha_kD(\mathcal{T}(Q_k) - Q^*)  + \alpha_k(F_k - \bar{F}_k).
\end{align*}
Therefore, using the definition of $\Delta_k$, we obtain
\begin{align}
\Delta_{k+1} = (I - \alpha_kD)\Delta_k + \alpha_kD(\mathcal{T}(Q_k) - Q^*)  + \alpha_k(F_k - \bar{F}_k). \label{eqn_2}
\end{align}
Let $V_k(s) := \max_a Q_k(s,a) = Q_k(s,\pi_k(s))$ and define $P^{\pi} \in \mathbb{R}^{|\mathcal{S}||\mathcal{A}|\times |\mathcal{S}||\mathcal{A}|}$ such that $P^{\pi}Q=P(\pi Q)$, we observe
\begin{align*}
\alpha_kD(\mathcal{T}(Q_k) - Q^*) &= \alpha_kD( (r+ \gamma PV_k) - (r + \gamma PV^*))\\
&=\alpha_k \gamma D(PV_k -PV^*)\\
&=\alpha_k \gamma D(P^{\pi_k}Q_k - P^{\pi^*}Q^*)\\
&=\alpha_k \gamma D(P^{\pi_k}Q_k - P^{\pi^*} Q_k + P^{\pi^*} Q_k - P^{\pi^*}Q^*)\\
&=\alpha_k \gamma D(P^{\pi_k}-P^{\pi^*} )Q_k  + \alpha_k \gamma D P^{\pi^*} (Q_k - Q^*).
\end{align*}
Thus, \cref{eqn_1} holds by substituting the above expression into \cref{eqn_2}.

Next, we prove $\Delta^{\downarrow}_k \leq \Delta_k \leq \Delta^{\uparrow}_k$ by induction. The base case $k=0$ holds by initialization. Suppose the statement holds at $k$. We observe that, since $\alpha_k$ and the entries in matrices $D$ and $P^{\pi^*}$ are all bounded between $0$ and $1$, the entries in matrix $I - \alpha_kD + \alpha_k \gamma D P^{\pi^*}$ are nonnegative. Consequently,
\begin{align*}
(I - \alpha_kD + \alpha_k \gamma D P^{\pi^*}) \Delta^{\downarrow}_{k} \leq (I - \alpha_kD + \alpha_k \gamma D P^{\pi^*}) \Delta_{k} \leq (I - \alpha_kD + \alpha_k \gamma D P^{\pi^*}) \Delta^{\uparrow}_{k}.
\end{align*}
We now have
\begin{align*}
\Delta^{\downarrow}_{k+1} &= (I - \alpha_kD + \alpha_k \gamma D P^{\pi^*})\Delta^{\downarrow}_k  + \alpha_k(F_k - \bar{F}_k)\\
&\leq (I - \alpha_kD + \alpha_k \gamma D P^{\pi^*})\Delta_k  + \alpha_k(F_k - \bar{F}_k)\\
&\leq (I - \alpha_kD + \alpha_k \gamma D P^{\pi^*})\Delta_k + \alpha_k \gamma D(P^{\pi_k}-P^{\pi^*} )Q_k  + \alpha_k(F_k - \bar{F}_k)\\
&=\Delta_{k+1},
\end{align*}
where the last inequality holds because $(P^{\pi_k}-P^{\pi^*} )Q_k \geq 0$, as $\pi_k$ is greedy w.r.t. $Q_k$. We remark that $\pi_k$ is the greedy policy w.r.t. $Q_k$ over all states, as implied by the definition of the Bellman optimality operator $\mathcal{T}$. Next, we have
\begin{align*}
\Delta_{k+1} &= (I - \alpha_kD + \alpha_k \gamma D P^{\pi^*})\Delta_k + \alpha_k \gamma D(P^{\pi_k}-P^{\pi^*} )Q_k  + \alpha_k(F_k - \bar{F}_k)\\
&\leq (I - \alpha_kD + \alpha_k \gamma D P^{\pi^*})\Delta^{\uparrow}_k + \alpha_k \gamma D(P^{\pi_k}-P^{\pi^*} )Q_k  + \alpha_k(F_k - \bar{F}_k)\\
&= (I - \alpha_kD + \alpha_k \gamma D P^{\pi^*})\Delta^{\uparrow}_k + \alpha_k \gamma D(P^{\pi_k}-P^{\pi^*} )\Delta_k  + \alpha_k \gamma D(P^{\pi_k}-P^{\pi^*} )Q^* + \alpha_k(F_k - \bar{F}_k)\\
&\leq (I - \alpha_kD + \alpha_k \gamma D P^{\pi^*})\Delta^{\uparrow}_k + \alpha_k \gamma D(P^{\pi_k}-P^{\pi^*} )\Delta_k +\alpha_k(F_k - \bar{F}_k)\\
&= \Delta^{\uparrow}_{k+1},
\end{align*}
where the last inequality holds because $(P^{\pi_k}-P^{\pi^*} )Q^* \leq 0$, as $\pi^*$ is greedy w.r.t. $Q^*$. Thus, the statement holds at $k+1$, which completes the proof.
\end{proof}

\subsection{Proof of Theorem \ref{thm_CLT}}

\begin{proof}
We first recall
\begin{align*}
\Delta^{\uparrow}_{k+1} = (I - \alpha_kD + \alpha_k \gamma D P^{\pi^*})\Delta^{\uparrow}_k + \alpha_k \gamma D(P^{\pi_k}-P^{\pi^*} )\Delta_k  + \alpha_k(F_k - \bar{F}_k).
\end{align*}
Denoting $A=D-\gamma DP^{\pi^*}$, $Z_k = \gamma D(P^{\pi_k}-P^{\pi^*} )\Delta_k$, and $Z_k^{\prime} = F_k - \bar{F}_k$, by recursion we have
\begin{align*}
\Delta^{\uparrow}_{k+1} = \prod_{i=0}^k(I - \alpha_i A)\Delta_0 + \sum_{i=0}^k \left( \prod_{j=i+1}^k (I - \alpha_j A) \right)\alpha_i (Z_i + Z_i^{\prime}).
\end{align*}
Thus,
\begin{align*}
\sum_{k=1}^K\Delta^{\uparrow}_{k} &= \sum_{k=1}^K \prod_{i=0}^{k-1}(I - \alpha_i A)\Delta_0 +\sum_{k=1}^K \sum_{i=0}^{k-1} \left( \prod_{j=i+1}^{k-1} (I - \alpha_j A) \right)\alpha_i (Z_i + Z_i^{\prime})\\
&= \sum_{k=1}^K \prod_{i=0}^{k-1}(I - \alpha_i A)\Delta_0 +\sum_{i=0}^{K-1} \alpha_i \sum_{k=i+1}^{K} \left( \prod_{j=i+1}^{k-1} (I - \alpha_j A) \right) (Z_i + Z_i^{\prime}).
\end{align*}
Denote $\Psi_i^K= \alpha_i \sum_{k=i+1}^{K} \left( \prod_{j=i+1}^{k-1} (I - \alpha_j A) \right)$. We further expand:
\begin{align}
&\sum_{k=1}^K\Delta^{\uparrow}_{k} \nonumber \\
&= \sum_{k=1}^K \prod_{i=0}^{k-1}(I - \alpha_i A)\Delta_0 +\sum_{i=0}^{K-1} \Psi_i^K (Z_i + Z_i^{\prime}) \nonumber \\
&= \sum_{k=1}^K \prod_{i=0}^{k-1}(I - \alpha_i A)\Delta_0 +\sum_{i=0}^{K-1} A^{-1}(Z_i + Z_i^{\prime}) +\sum_{i=0}^{K-1} (\Psi_i^K- A^{-1}) (Z_i + Z_i^{\prime}) \nonumber \\
&= \underbrace{\sum_{k=1}^K \prod_{i=0}^{k-1}(I - \alpha_i A)\Delta_0}_{\text{Term (1)}} + \underbrace{\sum_{i=0}^{K-1} A^{-1}Z_i}_{\text{Term (2)}} + \underbrace{\sum_{i=0}^{K-1} A^{-1} Z_i^{\prime}}_{\text{Term (3)}} + \underbrace{\sum_{i=0}^{K-1} (\Psi_i^K- A^{-1}) Z_i}_{\text{Term (4)}}
+ \underbrace{\sum_{i=0}^{K-1} (\Psi_i^K- A^{-1}) Z_i^{\prime}}_{\text{Term (5)}}. \label{main_decom}
\end{align}
\textbf{Bounding Term (1).} By applying Lemma \ref{lem_t1} and using the bound $\| \Delta_0 \|_{\infty} \leq \frac{1}{1-\gamma}$, we have $\| K^{-\frac{1}{2}} \text{Term (1)} \|_{\infty} \leq O\left( K^{-\frac{1}{2}} \rho^{\frac{-1}{1-\beta}}(1-\gamma)^{\frac{\beta-2}{1-\beta}} \right)$.\\
\textbf{Bounding Term (2).} We first expand the expression
\begin{align*}
A^{-1}Z_i = (D(I-\gamma P^{\pi^*}))^{-1}\gamma D(P^{\pi_i}-P^{\pi^*} )(Q_i-Q^*).
\end{align*}
Denoting $\rho:=\underset{(s,a)\in\mathcal{S}\times \mathcal{A}}{\min} p(s,a)$, we observe
\begin{align}
\|A^{-1}\|_{\infty} = \|(D(I-\gamma P^{\pi^*}))^{-1} \|_{\infty} \leq \frac{1}{(1-\gamma)\rho} \label{eqn_An1}
\end{align}
and by Assumption \ref{asm_1} and Lemma \ref{lem_t4},
\begin{align}
\| Z_i \|_{\infty} \leq \|(P^{\pi_i}-P^{\pi^*} )(Q_i-Q^*)\|_{\infty} = L\|Q_i-Q^* \|^2_{\infty} \leq O\left(\frac{t_iL}{\rho(1-\gamma)^2 i}\right) \label{eqn_Zi}
\end{align}
where $t_i$ is the mixing time. Thus,
\begin{align*}
\left\| \frac{1}{\sqrt{K}} \sum_{i=0}^{K-1} A^{-1}Z_i \right\|_{\infty} &\leq \frac{1}{\sqrt{K}} \sum_{i=1}^{K} O\left( \frac{t_iL}{i(1-\gamma)^2\rho}\right) \leq \frac{1}{\sqrt{K}} \cdot O\left( \frac{t_{\max}L}{(1-\gamma)^2\rho}\right) \cdot \sum_{i=1}^{K} \frac{1}{i} \\
&\leq \tilde{O}\left( \frac{L}{\sqrt{K}(1-\gamma)^2\rho}\right).
\end{align*}
\textbf{Decomposing Term (3).} We now analyze the Markovian noise term 
\begin{align*}
\sum_{i=0}^{K-1} A^{-1} Z_i^{\prime} = \sum_{i=0}^{K-1} A^{-1} (F_i - \bar{F}_i) = \sum_{i=0}^{K-1} A^{-1} (F(Q_i, Y_i) - \mathbb{E}[F(Q_i, Y_i)]).
\end{align*}
We decompose this term into two parts, where the first part has a bounded norm and the second part is a bounded martingale difference sequence. To this end, we use the Poisson equation technique \citep{glynn1996liapounov, douc2018markov, makowski2002poisson, chen2020explicit} to transform the Markovian noise into a martingale difference sequence. By a standard use of the technique \citep{douc2018markov}, we know there exists a solution $X_{k}: \tilde{\mathcal{S}} \rightarrow \mathbb{R}^{|\mathcal{S}|\times|\mathcal{A}|}$ to the following Poisson’s equation for all $k\in [K]$,
\begin{align*}
F(Q_k, Y_k) - \mathbb{E}[F(Q_k, Y_k)] = X_k(y_k) - \mathbb{E}[X_k(Y_{k+1})| Y_{k}=y_k].
\end{align*}
For $i\in \tilde{\mathcal{S}}$, the closed form of the solution $X_k(i)$ is given by 
\begin{align*}
X_k(i) = \sum_{j\in \tilde{\mathcal{S}}} [I-\tilde{P}- {\bf{1}} \tilde{\mu}^{\top}]^{-1}(i,j)(F(Q_k,i)-\bar{F}_k).
\end{align*}
Under Assumption \ref{asm_2}, there exists a constant $c_0 >0 $ and $\kappa \in (0,1)$ such that 
\begin{align*}
\max_{i \in \tilde{S}}\| \tilde{P}^t(i,\cdot) - \tilde{\mu}(\cdot) \|_{\mathrm{TV}} \leq c_0 \kappa^t.
\end{align*}
We now state two important properties for $X_k$. The first is a boundedness property that $\|X_k\|_{\infty} \leq O(\frac{1}{(1-\gamma)(1-\kappa)})$, which follows directly from the above results. Next, we prove Lipschitzness by showing that
\begin{align*}
\| X_{k}(i) - X_{k^{\prime}}(i) \|_{\infty} &=  \| \sum_{j\in \tilde{\mathcal{S}}} [I-\tilde{P}- {\bf{1}} \tilde{\mu}^{\top}]^{-1}(i,j)(F(Q_k,i)-F(Q_{k^{\prime}},i)) \|_{\infty}\\
&\leq \frac{c_0}{1-\kappa} \| F(Q_k,i)-F(Q_{k^{\prime}},i) \|_{\infty}\\
&\leq \frac{2c_0}{1-\kappa} \| Q_k - Q_{k^{\prime}} \|_{\infty}. \tag{By Lemma \ref{lem_lipF}}
\end{align*}
We now decompose Term (3),
\begin{align*}
&\sum_{k=0}^{K-1} A^{-1} (F(Q_k, Y_k) - \mathbb{E}[F(Q_k, Y_k)])= \sum_{k=0}^{K-1} A^{-1} (X_k(Y_k) - \mathbb{E}[X_k(Y_{k+1})| Y_{k}])\\
&=\sum_{k=0}^{K-1} A^{-1} (X_k(Y_{k}) - X_{k+1}(Y_{k+1}) + X_{k+1}(Y_{k+1})  - X_{k}(Y_{k+1}) + X_{k}(Y_{k+1}) - \mathbb{E}[X_k(Y_{k+1})| Y_{k}])\\
&= \underbrace{A^{-1} (X_0(Y_{0}) - X_K(Y_{K}))}_{\text{Term (3a)}} + \underbrace{\sum_{k=0}^{K-1} A^{-1} ( X_{k+1}(Y_{k+1})  - X_{k}(Y_{k+1}))}_{\text{Term (3b)}} \\
&\indent + \underbrace{ \sum_{k=0}^{K-1} A^{-1}( X_{k}(Y_{k+1}) - \mathbb{E}[X_k(Y_{k+1})| Y_{k}])}_{\text{Term (3c)}}.
\end{align*}
For Term (3a), note that $\|A^{-1} (X_0(Y_{0}) - X_K(Y_{K}))\|_{\infty} \leq O(\frac{1}{(1-\gamma)^2(1-\kappa)\rho})$ by boundedness property. By the Lipschitzness property, for Term (3b) we obtain
\begin{align*}
&\left\| \sum_{k=0}^{K-1} A^{-1} ( X_{k+1}(Y_{k+1})  - X_{k}(Y_{k+1})) \right\|_{\infty} \leq \frac{1}{(1-\gamma)\rho}\sum_{k=0}^{K-1}  \left\|  X_{k+1}(Y_{k+1})  - X_{k}(Y_{k+1})\right\|_{\infty}\\
&\leq \frac{2c_0}{(1-\gamma)(1-\kappa)\rho}\sum_{k=0}^{K-1}  \left\|  Q_{k+1}  - Q_{k}\right\|_{\infty} = \frac{2c_0}{(1-\gamma)(1-\kappa)\rho}\sum_{k=0}^{K-1}  \left\| \alpha_{k}( F_k - Q_k) \right\|_{\infty}\\
&\leq O\left(\frac{1}{(1-\gamma)^2(1-\kappa)\rho}\right) \sum_{k=0}^{K-1} \frac{1}{(k+b)^{\beta}} =  O\left(\frac{K^{1-\beta}}{(1-\gamma)^2(1-\kappa)\rho}\right).
\end{align*}
We have analyzed the first two terms. We defer the analysis of Term (3c) to the end of the proof.\\
\textbf{Bounding Term (4).} By combining \cref{eqn_psiA} and \cref{eqn_Zi}, we have
\begin{align*}
&\left\| \frac{1}{\sqrt{K}} \sum_{i=0}^{K-1} (\Psi_i^K- A^{-1}) Z_i \right\|_{\infty} \\
&\leq \frac{1}{\sqrt{K}}\sum_{i=1}^{K} O\left( \frac{1}{i(\rho(1-\gamma))^{\frac{2-\beta}{1-\beta}}}  + \frac{(i-1)^{\beta}}{i\rho^2(1-\gamma)^2} + \frac{(1-\rho(1-\gamma)\alpha_K)^{K-i+1}}{\rho(1-\gamma)}  \right)\cdot \frac{1}{\rho (1-\gamma)^2i}\\
&\leq \tilde{O}\left( \frac{1}{\sqrt{K} \rho^{\frac{3-2\beta}{1-\beta}}(1-\gamma)^{\frac{4-3\beta}{1-\beta}}} \right).
\end{align*}
\textbf{Bounding Term (5).} 
Similarly to Term (3), we have
\begin{align*}
&\sum_{k=0}^{K-1} (\Psi_k^K- A^{-1}) Z_k^{\prime} = \sum_{k=0}^{K-1} (\Psi_k^K- A^{-1}) (X_k(Y_k) - \mathbb{E}[X_k(Y_{k+1})| Y_{k}])\\
&=\sum_{k=0}^{K-1} (\Psi_k^K- A^{-1}) (X_k(Y_{k}) - X_{k+1}(Y_{k+1}) + X_{k+1}(Y_{k+1})  - X_{k}(Y_{k+1}) + X_{k}(Y_{k+1}) - \mathbb{E}[X_k(Y_{k+1})| Y_{k}])\\
&= \underbrace{\sum_{k=0}^{K-1} (\Psi_k^K- A^{-1}) (X_k(Y_{k}) - X_{k+1}(Y_{k+1}))}_{\text{Term (5a)}} + \underbrace{\sum_{k=0}^{K-1} (\Psi_k^K- A^{-1})( X_{k+1}(Y_{k+1})  - X_{k}(Y_{k+1}))}_{\text{Term (5b)}}\\
&\indent + \underbrace{\sum_{k=0}^{K-1} (\Psi_k^K- A^{-1})( X_{k}(Y_{k+1}) - \mathbb{E}[X_k(Y_{k+1})| Y_{k}])}_{\text{Term (5c)}}
\end{align*}
Now we analyze each term individually. For Term (5a), 
\begin{align*}
&\sum_{k=0}^{K-1} (\Psi_k^K- A^{-1}) (X_k(Y_{k}) - X_{k+1}(Y_{k+1})) \\
&= \sum_{k=0}^{K-1} \left[ (\Psi_k^K- A^{-1}) X_k(Y_{k}) -  (\Psi_{k+1}^K- A^{-1}) X_{k+1}(Y_{k+1}) \right.\\
&\indent \left. + (\Psi_{k+1}^K- A^{-1}) X_{k+1}(Y_{k+1}) -  (\Psi_k^K- A^{-1}) X_{k+1}(Y_{k+1})\right]\\
&= (\Psi_0^K- A^{-1}) X_0(Y_{0}) -  (\Psi_{K}^K- A^{-1}) X_{K}(Y_{K}) + \sum_{k=0}^{K-1} \left[ (\Psi_{k+1}^K - \Psi_{k}^K) X_{k+1}(Y_{k+1})\right].
\end{align*}
By the boundedness of $X_k$, \cref{eqn_psiA}, and Lemma \ref{lem_difpsi}, we obtain
\begin{align*}
\| \text{Term (5a)} \|_{\infty} &\leq O\left(\frac{1}{(1-\gamma)(\rho(1-\gamma))^{\frac{2-\beta}{1-\beta}}} + \frac{1}{1-\gamma} \sum_{k=1}^{K} \frac{1}{k^{\beta}} \right)\\
&\leq O\left( \frac{1}{(1-\gamma)(\rho(1-\gamma))^{\frac{2-\beta}{1-\beta}}}+\frac{K^{1-\beta}}{1-\gamma} \right).
\end{align*}
By \cref{eqn_psiA} and the Lipschitzness property of $X_k$, we have
\begin{align*}
\| \text{Term (5b)} \|_{\infty} &\leq \sum_{k=1}^K O\left( \frac{1}{k(\rho(1-\gamma))^{\frac{2-\beta}{1-\beta}}}  + \frac{(k-1)^{\beta}}{k\rho^2(1-\gamma)^2} + \frac{(1-\rho(1-\gamma)\alpha_K)^{K-k+1}}{\rho(1-\gamma)}  \right) \cdot \frac{1}{k^{\beta}}\\
&\leq \tilde{O}\left(\frac{1}{(\rho(1-\gamma))^{\frac{2-\beta}{1-\beta}}} + \frac{K^{1-\beta}}{\rho(1-\gamma)} \right).
\end{align*}
Next, we bound $\left\| \frac{1}{\sqrt{K}} \mathbb{E}[\text{Term (5c)]} \right\|_{\infty}$. We first note that $\{M_k,\mathcal{F}_k\}_{k\in[K]}$ is a martingale difference sequence where $\{M_k\}_{k\in[K]}:=\{X_k(Y_{k+1}) - \mathbb{E}[X_k(Y_{k+1})|Y_k]\}_{k\in[K]}$ and $\mathcal{F}_k$ is $\sigma$-field generated by all randomness until iteration $k$. Thus, by the martingale difference property we have $\mathbb{E}[M_k|\mathcal{F}_{k-1}]=0$ and $\mathbb{E}[\langle M_i,M_j\rangle] = \mathbb{E}[\langle M_i, \mathbb{E}[M_j|\mathcal{F}_{j-1}]\rangle] =0$ for $i\neq j$. This leads to
\begin{align*}
&\left\| \frac{1}{\sqrt{K}} \mathbb{E}[\text{Term (5c)]} \right\|^2_{\infty} \\
&=\left\| \frac{1}{\sqrt{K}}  \mathbb{E} \sum_{k=0}^{K-1} (\Psi_k^K- A^{-1})( X_{k}(Y_{k+1}) - \mathbb{E}[X_k(Y_{k+1})| Y_{k}]) \right\|^2_{\infty}\\
&\leq \frac{1}{(1-\gamma)^2K} \sum_{k=0}^{K-1} \|\Psi_k^K- A^{-1} \|^2_{\infty}\\
&\leq \frac{1}{(1-\gamma)^2K} \sum_{i=1}^K O\left( \frac{1}{i(\rho(1-\gamma))^{\frac{2-\beta}{1-\beta}}}  + \frac{(i-1)^{\beta}}{i\rho^2(1-\gamma)^2} + \frac{(1-\rho(1-\gamma)\alpha_K)^{K-i+1}}{\rho(1-\gamma)}  \right)^2 \tag{by \cref{eqn_psiA}}\\
&\leq \frac{1}{(1-\gamma)^2} \cdot \tilde{O}\left( \frac{1}{K(\rho(1-\gamma))^{\frac{4-2\beta}{1-\beta}}} + \frac{1}{K^{1-\beta}\rho^4(1-\gamma)^4} \right).
\end{align*}
Thus,
\begin{align*}
\left\| \frac{1}{\sqrt{K}} \mathbb{E}[\text{Term (5c)]} \right\|_{\infty} \leq \tilde{O}\left( \frac{1}{{K}^{1/2}(1-\gamma)(\rho(1-\gamma))^{\frac{2-\beta}{1-\beta}}} + \frac{1}{K^{1/2-\beta/2}\rho^2(1-\gamma)^3} \right).
\end{align*}
Combining three terms, we have
\begin{align*}
\left\| \frac{1}{\sqrt{K}} \mathbb{E}[\text{Term (5)]} \right\|_{\infty} \leq \tilde{O}\left( \frac{1}{{K}^{1/2}(1-\gamma)(\rho(1-\gamma))^{\frac{2-\beta}{1-\beta}}} + \frac{1}{K^{1/2-\beta/2}\rho^2(1-\gamma)^3} + \frac{1}{K^{\beta-1/2}\rho(1-\gamma)} \right).
\end{align*}
\textbf{Putting Everything Together.} At this stage, we have decomposed $\sum_{k=1}^K\Delta^{\uparrow}_{k}$ into six components $\{\phi_i \}_{i=1}^6$, where $\phi_i$ corresponds to Term (i) for $i=1,2,4,5$ and Term (3) is further split into $\phi_3 = A^{-1} (X_0(Y_{0}) - X_K(Y_{K})) $ and $\phi_6=  \sum_{i=0}^{K-1} A^{-1}(X_k(Y_{k+1}) - \mathbb{E}[X_k(Y_{k+1})| Y_{k}])$. Accordingly,
\begin{align}
\sum_{k=1}^K\Delta^{\uparrow}_{k} = \sum_{i=1}^6 \phi_i = \sum_{i=1}^5 \phi_i + \sum_{k=0}^{K-1} A^{-1}(X_k(Y_{k+1}) - \mathbb{E}[X_k(Y_{k+1})| Y_{k}]) \label{eqn_decompose}
\end{align}
where $\phi_6$ is a bounded martingale difference sequence.
Note we have also established bounds for $\{\phi_i \}_{i=1}^5$. Therefore, to establish CLTs for the averaged $Q$-learning iterates, we can apply any suitable known martingale CLTs. To proceed, we choose the non-asymptotic martingale CLT given in \citet{srikant2024rates}. We prove in Lemma \ref{lem_mds} that $\mathcal{W}_1\left( \frac{1}{\sqrt{K}} \phi_6,  (A^{-1}\Sigma A^{-\top})^{1/2}\mathcal{N}(0,I)) \right) \leq O\left( \frac{1}{(1-\gamma)\rho K^{\beta/2}} \right)$. Note that 
\begin{align*}
&\mathcal{W}_1 \left( \frac{1}{\sqrt{K}} \sum_{k=1}^K\Delta^{\uparrow}_{k}, (A^{-1}\Sigma A^{-\top})^{1/2}\mathcal{N}(0,I)  \right) \\
&= \sup_{h 
\in \mathrm{Lip}_1} \mathbb{E}\left[ h\left( \frac{1}{\sqrt{K}} \sum_{k=1}^K\Delta^{\uparrow}_{k} \right) -h((A^{-1}\Sigma A^{-\top})^{1/2}\mathcal{N}(0,I)) \right].
\end{align*}
For any $h\in\mathrm{Lip}_1$, we have
\begin{align*}
&\mathbb{E}\left[ h\left( \frac{1}{\sqrt{K}} \sum_{k=1}^K\Delta^{\uparrow}_{k} \right) -h((A^{-1}\Sigma A^{-\top})^{1/2}\mathcal{N}(0,I)) \right]\\
&=\mathbb{E}\left[ h\left( \frac{1}{\sqrt{K}}\sum_{i=1}^6 \phi_i \right) -h((A^{-1}\Sigma A^{-\top})^{1/2}\mathcal{N}(0,I)) \right]\\
&=\underbrace{\sum_{i=1}^5 \mathbb{E}\left[ h\left( \frac{1}{\sqrt{K}}\sum_{k=i}^6 \phi_k \right) - h\left( \frac{1}{\sqrt{K}} \sum_{j=i+1}^6 \phi_j \right) \right]}_{Ta} + \underbrace{\mathbb{E}\left[h\left(\frac{1}{\sqrt{K}}\phi_6\right)\right] - h((A^{-1}\Sigma A^{-\top})^{1/2}\mathcal{N}(0,I))}_{T_b}.
\end{align*}
By Lemma \ref{lem_mds}, we have $T_b \leq O\left( ((1-\gamma)\rho)^{-2-\beta}K^{-\beta/2} \right).$ To bound $T_a$, by combining all bounds analyzed above, merging alike terms, and ignoring constants, we have
\begin{align*}
T_a &\leq \sum_{i=1}^5 \mathbb{E}\left\|\frac{1}{\sqrt{K}} \phi_i\right\|_2 \leq  \sum_{i=1}^5 \sqrt{|\mathcal{S}||\mathcal{A}|}\mathbb{E}\left\| \frac{1}{\sqrt{K}} \phi_i \right\|_{\infty} \\
&\leq \frac{\sqrt{{|\mathcal{S}||\mathcal{A}|}}}{\rho(1-\gamma)^2} \cdot \tilde{O}\left( \frac{1}{K^{1/2} (\rho(1-\gamma))^{\frac{2-\beta}{1-\beta}}} + \frac{1}{K^{1/2-\beta/2}\rho(1-\gamma)} + \frac{1}{K^{\beta-1/2}} \right).
\end{align*}
Thus, we have shown that
\begin{align*}
\mathcal{W}_1 \left( \frac{1}{\sqrt{K}} \sum_{k=1}^K\Delta^{\uparrow}_{k}, \tilde{\mathcal{N}}  \right) \leq R(K,\rho,1-\gamma,|\mathcal{S}|,|\mathcal{A}|)
\end{align*}
where $\tilde{\mathcal{N}}:=(A^{-1}\Sigma A^{-\top})^{1/2}\mathcal{N}(0,I)$ and
\begin{align*}
R(K,\rho,1-\gamma,|\mathcal{S}|,|\mathcal{A}|) &:= \frac{\sqrt{{|\mathcal{S}||\mathcal{A}|}}}{\rho(1-\gamma)^2} \cdot \tilde{O}\left( \frac{1}{K^{1/2} (\rho(1-\gamma))^{\frac{2-\beta}{1-\beta}}} \right. \\
&\left. \indent + \frac{1}{K^{1/2-\beta/2}\rho(1-\gamma)} + \frac{1}{K^{\beta-1/2}} + \frac{1}{K^{\beta/2}\rho^{1+\beta}(1-\gamma)^{\beta}} \right).
\end{align*}
Next, we show that a similar convergence also holds for $\Delta^{\downarrow}_{k}$. By Lemma \ref{lem_1}, we know
\begin{align*}
\Delta^{\downarrow}_{k+1} = (I - \alpha_kD + \alpha_k \gamma D P^{\pi^*})\Delta^{\downarrow}_k  + \alpha_k(F_k - \bar{F}_k).
\end{align*}
By a similar decomposition as in \cref{main_decom}, we obtain
\begin{align*}
\sum_{k=1}^K\Delta^{\downarrow}_{k} = \sum_{k=1}^K \prod_{i=0}^{k-1}(I - \alpha_i A)\Delta_0 + \sum_{i=0}^{K-1} A^{-1} Z_i^{\prime} + \sum_{i=0}^{K-1} (\Psi_i^K- A^{-1}) Z_i^{\prime}
\end{align*}
which matches Term (1), Term (3), and Term (5) in \cref{main_decom}. Thus, following the same steps as before,
\begin{align*}
\mathcal{W}_1 \left( \frac{1}{\sqrt{K}} \sum_{k=1}^K\Delta^{\downarrow}_{k}, \tilde{\mathcal{N}}  \right) = R(K,\rho,1-\gamma,|\mathcal{S}|,|\mathcal{A}|).
\end{align*}
By Lemma \ref{lem_1}, we have $\Delta^{\downarrow}_k \leq \Delta_k \leq \Delta^{\uparrow}_k$ for all $k \in [K]$. Therefore, we conclude that
\begin{align*}
\mathcal{W}_1 \left( \frac{1}{\sqrt{K}} \sum_{k=1}^K\Delta_{k}, \tilde{\mathcal{N}}  \right) &\leq \mathcal{W}_1 \left( \frac{1}{\sqrt{K}} \sum_{k=1}^K\Delta^{\uparrow}_{k}, \tilde{\mathcal{N}}  \right) + \mathcal{W}_1 \left( \frac{1}{\sqrt{K}} \sum_{k=1}^K\Delta^{\downarrow}_{k}, \tilde{\mathcal{N}}  \right) \\
&= R(K,\rho,1-\gamma,|\mathcal{S}|,|\mathcal{A}|).
\end{align*}
\end{proof}

\subsection{Martingale CLT}

\begin{theorem}[Restatement of Theorem 1 in \citet{srikant2024rates}]
\label{thm_mclt}
Let $\{ m_k \}_{k\geq 1}$ be a $d$-dimensional martingale difference sequence with respect to a filtration $\{ \mathcal{F}_k \}_{k\geq 0}$. Assume (i) $\mathbb{E}[\|m_k\|_2] \leq \infty$ and $\mathbb{E}[m_k|\mathcal{F}_{k-1}]=0$ for all $k \geq 1$; (ii) $\mathbb{E}[\| m_k\|_2^{2+\beta}]$ exists almost surely for all $k \geq 1$ and some $\beta \in (0,1)$ and (iii) $\Sigma_{k}= \mathbb{E}[m_k m_k^{\top}|\mathcal{F}_{k-1}]$ exists and further assume that $\lim_{n\rightarrow \infty} (\Sigma_1+\cdots+\Sigma_{n})/n = \Sigma_{\infty}$ almost surely for some positive definite $\Sigma_{\infty}$. It follows that
\begin{align*}
\mathcal{W}_1\left( \sum_{k=1}^n \frac{m_k}{\sqrt{n}} , \Sigma_{\infty}^{1/2}\mathcal{N}(0,I)\right) &\leq \frac{1}{\sqrt{n}} \sum_{k=1}^n O\left( \frac{\|\Sigma_{\infty}^{1/2}\|_{\mathrm{op}}\mathbb{E}[\|\Sigma_{\infty}^{-1/2} m_k \|_2^{\beta+2} + \| \Sigma_{\infty}^{-1/2} m_k \|_2^{\beta}]}{(n-k+1)^{(1+\beta)/2}} \right. \\
&\left. \indent - \frac{1}{n-k+1} \mathrm{Tr}(M_k(\Sigma_{\infty}^{-1/2}\mathbb{E}[\Sigma_k]\Sigma_{\infty}^{-1/2}-I)) \right)
\end{align*}
where $M_k$ is a matrix with the property $\|M_k\|_{\mathrm{op}} \leq O(\sqrt{n-k+1} \| \Sigma_{\infty}^{1/2} \|_{\mathrm{op}})$.
\end{theorem}

\begin{lemma}
\label{lem_mds}
Under Assumption \ref{asm_2},
\begin{align*}
\mathcal{W}_1 \left( \frac{1}{\sqrt{K}}\sum_{k=1}^{K} A^{-1}( X_{k-1}(Y_{k}) - \mathbb{E}[X_{k-1}(Y_{k})| Y_{k-1}]), \tilde{\mathcal{N}} \right) \leq O\left( ((1-\gamma)\rho)^{-2-\beta}K^{-\beta/2} \right)
\end{align*}
where $\tilde{\mathcal{N}} = A^{-1} \Sigma A^{-\top} \mathcal{N}(0,I)$ and $\Sigma := \sum_{i,j\in \tilde{S}} \tilde{\mu}(i)  \tilde{P}(i,j) (X(j) - \mathbb{E}[X(Y_1)|Y_0=i])(X(j) - \mathbb{E}[X(Y_1)|Y_0=i])^{\top}.$
\end{lemma}

\begin{proof}
We first define $X: \tilde{S}\rightarrow \mathbb{R}^{|\mathcal{S}|\times|\mathcal{A}|}$ to be the solution of the following Poisson's equation,
\begin{align*}
F(Q^*, i) - \mathbb{E}[F(Q^*, i)] = X(i) - \mathbb{E}[X(Y_1)| Y_0=i] \ \ \text{for all}\ \ i\in \tilde{S}.
\end{align*}
Denote $\tilde{p}_t(i) := \mathbb{P}(Y_t=i)$. We further define the covariance matrix of the martingale noise characterized via the solution of Poisson’s equation and its asymptotic matrix by
\begin{align*}
\tilde{\Sigma}_{k} = \sum_{i,j\in \tilde{S}} \tilde{p}_k(i)  \tilde{P}(i,j) (X_k(j) - \mathbb{E}[X_k(Y_1)|Y_0=i])(X_k(j) - \mathbb{E}[X_k(Y_1)|Y_0=i])^{\top}
\end{align*}
and
\begin{align*}
\Sigma := \sum_{i,j\in \tilde{S}} \tilde{\mu}(i)  \tilde{P}(i,j) (X(j) - \mathbb{E}[X(Y_1)|Y_0=i])(X(j) - \mathbb{E}[X(Y_1)|Y_0=i])^{\top}.
\end{align*}
Now we can substitute $n = K$, $m_k = A^{-1}( X_{k-1}(Y_{k}) - \mathbb{E}[X_{k-1}(Y_{k})| Y_{k-1}])$, $\Sigma_{k} = A^{-1}\tilde{\Sigma}_{k}A^{-\top}$, and $\Sigma_{\infty}=A^{-1}\Sigma A^{-\top}$ into Theorem \ref{thm_mclt}. Note that under Assumption \ref{asm_2}, the three conditions in Theorem \ref{thm_mclt} are satisfied. The rest of the proof follows from the proof of Theorem 2 in \citet{srikant2024rates}, with necessary modifications to accommodate our setting. To conclude, with the substitutions, we have
\begin{align*}
\sum_{k=1}^n \frac{\|\Sigma_{\infty}^{1/2}\|_{\mathrm{op}}\mathbb{E}[\|\Sigma_{\infty}^{-1/2} m_k \|_2^{\beta+2} + \| \Sigma_{\infty}^{-1/2} m_k \|_2^{\beta}]}{(n-k+1)^{(1+\beta)/2}} \leq O\left(  n^{(1-\beta)/2} /((1-\gamma)\rho)^{2+\beta}\right)
\end{align*}
and
\begin{align*}
\sum_{k=1}^n\frac{1}{n-k+1} \mathrm{Tr}(M_k(\Sigma_{\infty}^{-1/2}\mathbb{E}[\Sigma_k]\Sigma_{\infty}^{-1/2}-I)) \leq O\left( 1/(1-\gamma)^2\rho^2 \right)
\end{align*}
which completes the proof.

\end{proof}

\section{Proof of Theorem \ref{thm_FCLT}}

Polish space is a separable and complete function space. It is a crucial structure for applying convergence in distribution results such as FCLT. Recall that we denote $\mathcal{D}[0,1]$ as the Skorokhod space. Equipped with the Skorokhod $J_1$ topology with a particular metric \citep{prokhorov1956convergence}, $\mathcal{D}[0,1]$ is a Polish space. We use $\overset{\mathrm{w}}{\rightarrow}$ to denote weak convergence for some sequence of random elements. To prove the theorem, we need the following result.

\begin{proposition}
\label{prop_conv}
For two random sequences $\{X_t\}_{t \geq 0}, \{Y_t\}_{t \geq 0} \subseteq \mathcal{D}[0,1]$ satisfying $\mathbb{E}[ {\sup}_{\kappa\in [0,1]}\|Y_T(\kappa)\| ]\to 0$ and $X_T \xrightarrow{\mathrm{w}} X$, we have $X_T + Y_T \xrightarrow{\mathrm{w}} X$.
\end{proposition}
Now we prove Theorem \ref{thm_FCLT}.
\begin{proof}
For $\zeta\in [0,1]$, by a similar decomposition as in \cref{eqn_decompose}, we have
\begin{align*}
\Phi_{K}^{\uparrow}(\zeta) &:=  \frac{1}{\sqrt{K}}\sum_{k=1}^{\lfloor \zeta K \rfloor}\Delta^{\uparrow}_{k} =  \frac{1}{\sqrt{K}} \sum_{i=1}^6 \phi_i(\zeta)\\
&= \frac{1}{\sqrt{K}} \sum_{i=1}^5 \phi_i(\zeta) +  \frac{1}{\sqrt{K}}\sum_{k=1}^{\lfloor \zeta K \rfloor} A^{-1}(X_{k-1}(Y_{k}) - \mathbb{E}[X_{k-1}(Y_{k})| Y_{k-1}]).
\end{align*}
From the proof of Theorem \ref{thm_CLT}, we know $\sup_{\zeta\in [0,1]}\left\| \frac{1}{\sqrt{K}}\phi_i(\zeta)\right\|_{\infty} = o(1)$ for $i \in \{1,2,3,4,5\}$. Let $X$ and $\Sigma$ as defined in the proof of Lemma \ref{lem_mds}. The following lemma, which establishes the FCLT for $\frac{1}{\sqrt{K}}\phi_6(\zeta)$, is a direct consequence of Theorem 4.2 in \citet{hall2014martingale}.
\begin{lemma}
For any $\zeta\in [0,1]$,
\begin{align*}
\frac{1}{\sqrt{K}}\sum_{k=1}^{\lfloor \zeta K \rfloor} A^{-1}(X_{k-1}(Y_{k}) - \mathbb{E}[X_{k-1}(Y_{k})| Y_{k-1}]) \overset{\mathrm{w}}{\rightarrow} (A^{-1}\Sigma A^{-\top})^{1/2}\textbf{B}(\zeta)
\end{align*}
where $\textbf{B}$ is the standard Brownian motion and 
$\Sigma := \sum_{i,j\in \tilde{S}} \tilde{\mu}(i)  \tilde{P}(i,j) (X(j) - \mathbb{E}[X(Y_1)|Y_0=i])(X(j) - \mathbb{E}[X(Y_1)|Y_0=i])^{\top}.$
\end{lemma}
Thus, we have $\frac{1}{\sqrt{K}}\phi_6(\cdot) \overset{\mathrm{w}}{\rightarrow} (A^{-1}\Sigma A^{-\top})^{1/2}\textbf{B}(\cdot)$. Besides, we observe that
\begin{align*}
\sup_{\zeta \in [0,1]}\left\|\Phi_{K}^{\uparrow}(\zeta) - \frac{1}{\sqrt{K}}\phi_6 (\zeta)\right\|_{\infty} \leq \sum_{i=1}^5 \sup_{\zeta \in [0,1]}\left\| \frac{1}{\sqrt{K}}\phi_i(\zeta)\right\|_{\infty} = o(1),
\end{align*}
which implies $\Phi_{K}^{\uparrow}(\cdot) \overset{\mathrm{w}}{\rightarrow} (A^{-1}\Sigma A^{-\top})^{1/2}\textbf{B}(\cdot)$ by Proposition \ref{prop_conv}. The FCLT for $\Phi_{K}^{\downarrow}(\cdot) := \frac{1}{\sqrt{K}}\sum_{k=1}^{\lfloor \cdot K \rfloor} \Delta_k^{\downarrow}$ can be established in the same way. Therefore, by the sandwich inequality, we have
\begin{align*}
&\sup_{\zeta \in [0,1]}\left\|\Phi_{K}(\zeta) - (A^{-1}\Sigma A^{-\top})^{1/2}\textbf{B}(\zeta)\right\|_{\infty} \\
&\leq \sup_{\zeta \in [0,1]}\left\|\Phi_{K}^{\uparrow}(\zeta) - (A^{-1}\Sigma A^{-\top})^{1/2}\textbf{B}(\zeta)\right\|_{\infty}  + \sup_{\zeta \in [0,1]}\left\|\Phi_{K}^{\downarrow}(\zeta) - (A^{-1}\Sigma A^{-\top})^{1/2}\textbf{B}(\zeta)\right\|_{\infty}  = o(1),
\end{align*}
which implies that $\Phi_{K}(\cdot) \overset{\mathrm{w}}{\rightarrow} (A^{-1}\Sigma A^{-\top})^{1/2}\textbf{B}(\cdot)$. This completes the proof.
    
\end{proof}

\section{Supporting Lemmas}

In this section we present several supporting lemmas. Lemma \ref{lem_t1} and \ref{lem_difpsi} analyze Term (1) and $\Psi_i^K$ appeared in the proof of Theorem \ref{thm_CLT}. Next, by leveraging the results in \citet{chen2021lyapunov}, Lemma \ref{lem_t4} gives a non-asymptotic convergence rate for $\Delta_k = Q_k - Q^*$ under asynchronous updates. Lastly, Lemma \ref{lem_lipF} provides a Lipschitz property for the operator $F(\cdot,s)$ defined in \cref{update_F}.

\begin{lemma}
\label{lem_t1}
Let $\alpha_i = \alpha{(i + b)^{-\beta}}$ for some problem-dependent constants $\alpha, b > 0$ and $\beta \in (0,1)$. Then the following bounds hold:
\begin{align*}
\left\| \sum_{k=1}^K \prod_{i=0}^{k-1}(I - \alpha_i A) \right\|_{\infty} &\leq O\left( {( \rho(1-\gamma))^{\frac{-1}{1-\beta}}} \right), \\
\sum_{i=1}^K\|\Psi_i^K - A^{-1}\|_{\infty} &\leq \tilde{O}\left( {(\rho(1-\gamma))^{\frac{\beta-2}{1-\beta}}}  + \frac{K^{\beta}}{\rho^2(1-\gamma)^2} \right).
\end{align*}
\end{lemma}

\begin{proof}
The analysis of polynomial step sizes has been well studied in prior work (see, e.g., \citet{polyak1992acceleration,srikant2024rates,li2023statistical}). However, due to a slightly modified choice of the step-size and the different update rule in the asynchronous setting, we provide a complete proof for the sake of completeness.
Recall that $\alpha_i = \alpha{(i+b)^{-\beta}}$ and $\rho:=\underset{(s,a)\in\mathcal{S}\times \mathcal{A}}{\min} p(s,a)$. We now have
\begin{align*}
\left\| \sum_{k=1}^K \prod_{i=0}^{k-1}(I - \alpha_i A) \right\|_{\infty} &= \left\|  \sum_{k=1}^K \prod_{i=0}^{k-1}(I - \alpha_i (D-\gamma DP^{\pi^*})) \right\|_{\infty}\\
&\leq \sum_{k=1}^K \prod_{i=0}^{k-1}( 1-\alpha_i \rho(1-\gamma))\\
&=  \sum_{k=1}^K \prod_{i=0}^{k-1}\left( 1- \frac{\alpha \rho(1-\gamma)}{(i+b)^{\beta}} \right)\\
&\leq  \sum_{k=1}^K \exp\left( -\alpha \rho(1-\gamma) \sum_{i=0}^{k-1}  (i+b)^{-\beta} \right) \tag{$1-x\leq\exp(-x)$}.\\
\intertext{For $\beta \in (0,1)$, we have $\sum_{i=0}^{k-1}  (i+b)^{-\beta} \geq \int_{0}^{k-1} (x+b)^{-\beta}dx=\frac{(k-1+b)^{1-\beta}+b^{1-\beta}}{1-\alpha}$, }
\left\| \sum_{k=1}^K \prod_{i=0}^{k-1}(I - \alpha_i A) \right\|_{\infty} &\leq \sum_{k=1}^K \exp\left( -\alpha \rho(1-\gamma) \frac{(k-1+b)^{1-\beta}+b^{1-\beta}}{1-\alpha} \right) \\
&\leq \int_{1}^{\infty} \exp\left( -\alpha \rho(1-\gamma) \frac{(k-1+b)^{1-\beta}+b^{1-\beta}}{1-\alpha} \right) dk
\intertext{by the change of variable $u=-\alpha \rho(1-\gamma) \frac{(k-1+b)^{1-\beta}+b^{1-\beta}}{1-\alpha}$,}
&\leq \frac{1}{\alpha \rho(1-\gamma)} \int_{0}^{\infty} \left( { \frac{(1-\beta)u}{\alpha \rho(1-\gamma)}+b^{1-\beta}} \right)^{\frac{\beta}{1-\beta}} \exp(-u) du\\
&\leq \frac{\max\{2^{\frac{\beta}{1-\beta}},1\}}{\alpha \rho(1-\gamma)} \int_{0}^{\infty} \left( { \left(\frac{(1-\beta)u}{\alpha \rho(1-\gamma)} \right)^{\frac{\beta}{1-\beta}} +b^{\beta}} \right) \exp(-u) du.
\intertext{Since $\int_0^{\infty}\exp(-u)du=1$ and $\int_0^{\infty}u^{\frac{\beta}{1-\beta}}\exp(-u)du=\Gamma(\frac{1}{1-\beta})\leq \frac{\sqrt{2\pi e}}{\sqrt{1-\beta}}(\frac{1}{1-\beta})^{\frac{\beta}{1-\beta}}$,}
\left\| \sum_{k=1}^K \prod_{i=0}^{k-1}(I - \alpha_i A) \right\|_{\infty} &\leq \frac{\max\{2^{\frac{\beta}{1-\beta}},1\}}{\alpha \rho(1-\gamma)} \left( { \left(\frac{1}{\alpha \rho(1-\gamma)} \right)^{\frac{\beta}{1-\beta}} \frac{\sqrt{2\pi e}}{\sqrt{1-\beta}} +b^{\beta}} \right) \\
&\leq O\left( \frac{1}{(\alpha \rho(1-\gamma))^{\frac{1}{1-\beta}}(1-\beta)^{\frac{1}{2}}} \right).
\end{align*}
Next, we prove the second part. Recall $\Psi_i^K= \alpha_i \sum_{k=i+1}^{K} \left( \prod_{j=i+1}^{k-1} (I - \alpha_j A) \right)$. Since $A^{-1}= \alpha_i^{-1}(I-(I-\alpha_i A))$, we have
\begin{align}
\Psi_i^K - A^{-1} &= (\Psi_i^K A - I)A^{-1} \nonumber \\
&= \left(\sum_{t=i+1}^K \left(\prod_{j=i+1}^{t-1} (I - \alpha_j A) - \prod_{j=i}^{t-1} (I - \alpha_j A)\right) A^{-1} - A^{-1}\right) \nonumber \\
&= \sum_{t=i+1}^K \left(\left(\prod_{j=i+1}^{t-1} (I - \alpha_j A) - \prod_{j=i}^{t-2} (I - \alpha_j A)\right) A^{-1}\right) - \left( \prod_{j=i}^{K} (I - \alpha_j A)\right) A^{-1} \nonumber \\
&= \underbrace{\sum_{t=i+1}^K (\alpha_i - \alpha_t) \prod_{j=i+1}^{t-2} (I - \alpha_j A)}_{T_1} - \underbrace{\left( \prod_{j=i}^{K} (I - \alpha_j A)\right) A^{-1}}_{T_2} .\label{step_ineq1}
\end{align}
For $T_1$, since $A=D-\gamma D P^{\pi^*}$ and $1-x \leq \exp(-x)$, we have
\begin{align}
\| T_1 \|_{\infty} &= \left\| \sum_{t=i+1}^K (\alpha_i - \alpha_t) \prod_{j=i+1}^{t-2} (I - \alpha_j A) \right\|_{\infty} \nonumber\\
&\leq \sum_{t=i+1}^K |\alpha_i - \alpha_t| \exp\left(-\sum_{j=i+1}^{t-2} \rho(1-\gamma)\alpha_j\right) \nonumber \\
&\leq \sum_{t=i+1}^K \sum_{k=i}^{t-1} |\alpha_{k+1} - \alpha_k| \exp\left(-\sum_{j=i+1}^{t-2} \rho(1-\gamma)\alpha_j\right). \nonumber
\intertext{Note that $\frac{\alpha_{k} - \alpha_{k+1}}{\alpha_{k}}=1-\left( 1- \frac{1}{k+1+b} \right)^{\beta} \leq 1- \exp(-\frac{\beta}{k+1+b}) \leq \frac{\beta}{k}$,}
\| T_1 \|_{\infty} &\leq \sum_{t=i+1}^K \sum_{k=i}^{t-1} \frac{\beta \alpha_k}{k} \exp\left(-\sum_{j=i+1}^{t-2} \rho(1-\gamma)\alpha_j\right) \nonumber \\
&\leq \frac{\beta}{\rho(1-\gamma)i} \sum_{t=i+1}^K \sum_{k=i}^{t-1}  \rho(1-\gamma) \alpha_k \exp\left(-\sum_{j=i+1}^{t-2} \rho(1-\gamma)\alpha_j\right) \nonumber \\
&\leq O\left( \frac{1}{i(\rho(1-\gamma))^{\frac{2-\beta}{1-\beta}}}  + \frac{(i-1)^{\beta}}{i\rho^2(1-\gamma)^2}   \right) \label{step_ineq2}.
\end{align}
For $T_2$, we obtain
\begin{align}
\| T_2 \|_{\infty} &= \left\| \left( \prod_{j=i}^{K} (I - \alpha_j A)\right) A^{-1} \right\|_{\infty} \nonumber \leq \| A^{-1}\|_{\infty} \prod_{j=i}^{K} \| I - \alpha_j A \|_{\infty} \nonumber \\
&\leq \frac{\prod_{j=i}^{K}(1-\rho(1-\gamma)\alpha_j)}{\rho(1-\gamma)}  \leq \frac{(1-\rho(1-\gamma)\alpha_K)^{K-i+1}}{\rho(1-\gamma)} \label{step_ineq3}.
\end{align}
Combining \cref{step_ineq1,step_ineq2,step_ineq3}, we have 
\begin{align}
\|\Psi_i^K - A^{-1}\|_{\infty}  = O\left( \frac{1}{i(\rho(1-\gamma))^{\frac{2-\beta}{1-\beta}}}  + \frac{(i-1)^{\beta}}{i\rho^2(1-\gamma)^2} + \frac{(1-\rho(1-\gamma)\alpha_K)^{K-i+1}}{\rho(1-\gamma)}  \right). \label{eqn_psiA}
\end{align}
Therefore,
\begin{align*}
\sum_{i=1}^K\|\Psi_i^K - A^{-1}\|_{\infty}  \leq  \tilde{O}\left( \frac{1}{(\rho(1-\gamma))^{\frac{2-\beta}{1-\beta}}}  + \frac{K^{\beta}}{\rho^2(1-\gamma)^2} \right). 
\end{align*}
\end{proof}

\begin{lemma}
Let $\alpha_i = \alpha{(i + b)^{-\beta}}$ for some problem-dependent constants $\alpha, b > 0$ and $\beta \in (0,1)$. It follows that
\begin{align*}
\sum_{t=i+1}^K \sum_{k=i}^{t-1}  \rho(1-\gamma) \alpha_k \exp\left(-\sum_{j=i+1}^{t-2} \rho(1-\gamma)\alpha_j\right) \leq O\left( \frac{1}{(\rho(1-\gamma))^{\frac{1}{1-\beta}}}  + \frac{(i-1)^{\beta}}{\rho(1-\gamma)}   \right).
\end{align*}
\end{lemma}

\begin{proof}
Since
\begin{align}
\sum_{k=i}^{t-1}  \rho(1-\gamma) \alpha_k \leq \frac{\rho\alpha(1-\gamma)}{1-\beta} ((t-1)^{1-\beta}-(i-1)^{1-\beta}) \leq \sum_{k=i-1}^{t-2}  \rho(1-\gamma) \alpha_k, \label{ine_2s}
\end{align}
we have
\begin{align*}
&\sum_{t=i+1}^K \sum_{k=i}^{t-1}  \rho(1-\gamma) \alpha_k \exp\left(-\sum_{j=i+1}^{t-2} \rho(1-\gamma)\alpha_j\right)\\
&= \sum_{t=i+1}^K \sum_{k=i}^{t-1}  \rho(1-\gamma) \alpha_k \exp\left(-\sum_{j=i-1}^{t-2} \rho(1-\gamma)\alpha_j\right)\exp\left( \rho(1-\gamma)(\alpha_i+\alpha_{i-1})\right)\\
&\leq e \sum_{t=i+1}^K \sum_{k=i}^{t-1}  \rho(1-\gamma) \alpha_k \exp\left(-\sum_{j=i-1}^{t-2} \rho(1-\gamma)\alpha_j\right)\\
&\leq e \sum_{t=i+1}^K u \exp\left(-u\right) \tag{let $u=\frac{\rho\alpha(1-\gamma)}{1-\beta} ((t-1)^{1-\beta}-(i-1)^{1-\beta})$ and by \cref{ine_2s}}\\
&\leq e \int_{0}^{\infty} u \exp\left(-u\right) \frac{1}{\rho\alpha(1-\gamma)} \left( \frac{1-\beta}{\rho\alpha(1-\gamma)} u + (i-1)^{1-\beta} \right)^{\frac{\beta}{1-\beta}} dt\\
&\leq \frac{e\max\{2^{\frac{\beta}{1-\beta}},2\}}{\rho\alpha(1-\gamma)} \int_{0}^{\infty} u \exp\left(-u\right) \left( \left(\frac{1-\beta}{\rho\alpha(1-\gamma)} u \right)^{\frac{\beta}{1-\beta}} + (i-1)^{\beta} \right) dt\\
&\leq \frac{e2^{\frac{1}{1-\beta}}}{\rho\alpha(1-\gamma)}  \left( \left(\frac{1-\beta}{\rho\alpha(1-\gamma)}  \right)^{\frac{\beta}{1-\beta}}\Gamma\left(1+\frac{1}{1-\beta}\right) + (i-1)^{\beta} \right) \\
&\leq O\left( \frac{1}{(\rho(1-\gamma))^{\frac{1}{1-\beta}}}  + \frac{(i-1)^{\beta}}{\rho(1-\gamma)}   \right).
\end{align*}
\end{proof}

\begin{lemma}
\label{lem_difpsi}
Let $\Psi_k^K= \alpha_k \sum_{i=k+1}^{K} \left( \prod_{j=k+1}^{i-1} (I - \alpha_j A) \right)$. For $k \geq (\rho(1-\gamma))^{-1/\beta(1-\beta)}$, we have
\begin{align*}
\left\| \Psi_{k+1}^K - \Psi_{k}^K \right\|_{\infty} \leq O\left( \frac{1}{k^{\beta}}\right).
\end{align*}
\end{lemma}

\begin{proof}
First, we have
\begin{align*}
\Psi_{k+1}^K - \Psi_{k}^K &=\alpha_{k+1} \sum_{i=k+1}^{K} \left( \prod_{j=k+1}^{i-1} (I - \alpha_j A) \right) - \alpha_{k} \sum_{i=k}^{K} \left( \prod_{j=k}^{i-1} (I - \alpha_j A) \right)\\
&= (\alpha_{k+1}-\alpha_k) \sum_{i=k+1}^{K} \left( \prod_{j=k+1}^{i-1} (I - \alpha_j A) \right) \\
&\indent + \alpha_{k} \left[\sum_{i=k+1}^{K} \left( \prod_{j=k+1}^{i-1} (I - \alpha_j A) \right) -\sum_{i=k}^{K} \left( \prod_{j=k}^{i-1} (I - \alpha_j A) \right) \right].
\end{align*}
For the first term above, by a similar analysis as in the proof of Lemma \ref{lem_t1} we have
\begin{align*}
\left\| (\alpha_{k+1}-\alpha_k) \sum_{i=k+1}^{K} \left( \prod_{j=k+1}^{i-1} (I - \alpha_j A) \right)\right\|_{\infty} &\leq \left( \frac{1}{k^{\beta}}-\frac{1}{(k+1)^{\beta}} \right) \cdot  O\left( \frac{1}{(\rho(1-\gamma))^{\frac{1}{1-\beta}}}  \right) \\
&\leq O\left( \frac{1}{k^{1+\beta}(\rho(1-\gamma))^{\frac{1}{1-\beta}}}  \right) \leq O\left(\frac{1}{k}\right)
\end{align*}
for $k \geq (\rho(1-\gamma))^{-1/\beta(1-\beta)}$. For the second term, we observe
\begin{align*}
\sum_{i=k+1}^{K} \left( \prod_{j=k+1}^{i-1} (I - \alpha_j A) \right) -\sum_{i=k}^{K} \left( \prod_{j=k}^{i-1} (I - \alpha_j A) \right) &= I + \sum_{i=k}^{K} \left( \prod_{j=k+1}^{i-1} (I - \alpha_j A) - \prod_{j=k}^{i-1} (I - \alpha_j A) \right)\\
&= I + \sum_{i=k}^{K} \left( \alpha_k A \prod_{j=k+1}^{i-1} (I - \alpha_j A) \right)\\
&= O(I)
\end{align*}
for $k \geq (\rho(1-\gamma))^{-1/\beta(1-\beta)}$. Thus, the second term is of order $\alpha_k$. Putting them together,
\begin{align*}
\left\| \Psi_{k+1}^K - \Psi_{k}^K \right\|_{\infty} \leq O\left( \frac{1}{k^{\beta}}\right).
\end{align*}
\end{proof}

The following lemmas provide finite-sample convergence guarantees of asynchronous Q-learning and Lipschitzness of the operator $F(\cdot,s)$ \citep{chen2021lyapunov}.

\begin{lemma}[Theorem B.1 in \citet{chen2021lyapunov}]
\label{lem_t4}
Let $\alpha_i = \alpha{(i + b)^{-\beta}}$ for some problem-dependent constants $\alpha, b > 0$ and $\beta \in (0,1)$. For the Q-learning updates in \cref{update_Q}, under Assumption \ref{asm_2}, we have $\mathbb{E}\| Q_k - Q^{*} \|_{\infty} \leq O\left( \sqrt{\frac{t_k}{(1-\gamma)^2\rho^2 k}} \right)$, where $t_k=O(\log(1/\alpha_k))$ denotes the mixing time.
\end{lemma}

\begin{lemma}[Proposition 3.1 in \citet{chen2021lyapunov}]
\label{lem_lipF}
For the operator $F(\cdot,\cdot)$ defined in \cref{update_F},  under Assumption \ref{asm_2}, we have $\|F(Q_1,s)-F(Q_2,s)\|_{\infty} \leq 2 \| Q_1 - Q_2\|_{\infty} $ for any $s \in \tilde{S}$.
\end{lemma}

\end{document}